\documentclass[a4paper]{article}
\usepackage{amsmath,amstext,amsgen,amsbsy,amsopn,amsfonts,amssymb}
\usepackage{easybmat}
\usepackage{graphics}
\usepackage[pdftex]{graphicx}
\usepackage{epsfig}
\usepackage{amsthm}
\usepackage{bm}
\usepackage{algorithm}
\usepackage{algorithmic}
\newtheorem{theorem}{Theorem}
\newtheorem{lemma}{Lemma}

\usepackage{wrapfig}
\usepackage{esvect}
\usepackage{multirow}
\usepackage{array}
\usepackage{xcolor}
\usepackage{hyperref}

\begin{document}
\large
\title{\textbf{Fast Communication-efficient Spectral Clustering Over Distributed Data}}
\author{
Donghui Yan$^{\dag\S}$, Yingjie Wang$^{\ddag\S }$, Jin Wang$^{\ddag\S }$, \\Guodong Wu$^{\P }$, Honggang Wang$^{\ddag\S}$
\vspace{0.1in}\\
$^\dag$Department of Mathematics and Program in Data Science\vspace{0.06in}\\
$^\ddag$Department of Electrical and Computer Engineering\vspace{0.06in}\\
$^{\S}$University of Massachusetts, Dartmouth, MA\vspace{0.1in}\\
$^{\P }$Lovelace Respiratory Research Institute, Albuquerque, NM 
\\
}
\date{\today}
\maketitle

\begin{abstract}
\noindent
The last decades have seen a surge of interests in distributed computing thanks to advances in clustered computing and 
big data technology. Existing distributed algorithms typically assume {\it all the data are already in one place}, and divide the data 
and conquer on multiple machines. However, it is increasingly often that the data are located at a number of distributed sites, 
and one wishes to compute over all the data with low communication overhead. For spectral clustering, we propose 
a novel framework that enables its computation over such distributed data, with ``minimal" communications while a major speedup 
in computation. The loss in accuracy is negligible compared to the non-distributed setting. Our approach 
allows local parallel computing at where the data are located, thus turns the distributed nature of the data into a blessing; the speedup 
is most substantial when the data are evenly distributed across sites. Experiments on synthetic and large UC Irvine 
datasets show almost no loss in accuracy with our approach while about 2x speedup under various settings with two distributed sites.
As the transmitted data need not be in their original form, our framework readily addresses the privacy concern for data sharing in 
distributed computing. 
\end{abstract}
\normalsize

\section{Introduction}
\label{section:introduction}
Spectral clustering \cite{ShiMalik2000, NgJordan2002, Luxburg2007, mynips2008, YanHuangJordan2009tech} 
refers to a class of clustering algorithms that work on the eigen-decomposition of the Gram matrix formed by the 
pairwise similarity of data points. It is widely acknowledged as the method of choice for clustering, due to its 
typically superior empirical performance, its flexibility in capturing a range of geometries such as nonlinearity 
and non-convexity \cite{NgJordan2002}, and its nice theoretical properties \cite{SpielmanTeng1996,chung97, GineKoltchinskii2006,LuxburgBelkin2008, CF}. Spectral clustering has been successfully 
applied to a wide spectrum of applications, including parallel processing \cite{Simon1991}, image segmentation 
\cite{ShiMalik2000,NystromSpectral}, robotics \cite{OlsonWalterTL2005,BrunskillKollarRoy2007}, web search \cite{CaiHeLMW2004}, spam detection, 
social network mining \cite{WhiteSmyth2005,MiaoSongZhangBai2008,KuruczBCL2009}, and market research \cite{ChangHuangWL2007} etc. 
\\
\\
Most existing spectral clustering algorithms are ``local" algorithms. That is, they assume all the data are in one place. 
Then either all computation are carried out on a single machine, or the data are split and delegated to a number of
nodes (a.k.a. machines or sites) for parallel computation \cite{ChenSongChang2011,HefeedaGao2012}. It is 
possible that the data may be initially stored at several distributed nodes but then pushed to a central server which 
split and re-distribute the data. Such a case is also treated as local, as all the data are in one place at a certain stage of 
the data processing. However, with the emergence of big data, it is increasingly often that the data of interest are {\it distributed}. 
That is, the data are stored at a number of distributed sites, as a result of diverse data collection channels or business 
operation etc. For example, a major retail vendor, such as {\it Walmart}, has sales data collected at {\it walmart.com}, or 
its {\it Walmart} stores, or its warehouse chains---{\it Sam's Club} etc. Such data from different sources are distributed as 
they are owned by different business groups, even all within the same corporation. Indeed there is no one central data center
at Walmart, rather the data are either housed at Walmart's Arkansas headquarter, or its e-commerce labs in the San Francisco 
Bay area, CA, due possibly to historical reason---Walmart was a pioneer of commercial vendors for a large scale adoption of 
digital technology in the late seventies to early eighties, while started its e-commerce business during the last decade. Many 
applications, however, would require data mining or learning of a global nature, that is, to use data from all the distributed sites, 
as that would be a more faithful representation of the real world or would yield better results due to a larger data size.
\\
\\
There are several challenges in the spectral clustering of data across distributed sites. The data at individual 
sites may be large. Many existing divide-and-conquer type of algorithms \cite{ChenSongChang2011,HefeedaGao2012} 
would collect data from distributed sites first, re-distribute the working load and then aggregate results 
computed at individual sites. The communication overhead will be high. Moreover, the data at individual sites 
may not have the same distribution. Additionally, 
the owners of the data at individual sites may not be willing to share either because the data entails value or the data 
itself may be too sensitive to share (not the focus of this work though). Now the question becomes, {\it can we carry out 
spectral clustering for data distributed over a number of sites without transmitting large amount data?} 
\\
\\
One possible solution is to carry out spectral clustering at individual sites and then ensemble. However, as the data 
distribution at individual sites may be very different, and to ensemble needs distributional information from 
individual sites which is often not easy to derive. Thus, an ensemble type of algorithm will not work, or, at least not in a 
straightforward way. Another possibility is to modify existing distributed algorithms, such as \cite{ChenSongChang2011,HefeedaGao2012}, 
but that would require support from the computing infrastructure---to enable a close coordination and frequent communication 
of intermediate results among individual nodes---which is not easy to implement and the solution may not be generally 
applicable.
\\
\\
Our proposed framework is based on the distortion minimizing local (DML) transformations \cite{YanHuangJordan2009tech}. 
The idea is to generate a small set of codewords (a.k.a. representative points) that preserve the underlying structure of 
the data at individual nodes. Such codewords are then transmitted to a central server (or any one of the nodes) where 
spectral clustering is carried out, and the result from spectral clustering will be populated to individual nodes so that the 
cluster membership of all points can be recovered according to correspondence information maintained at individual 
nodes. As the codewords from individual sites preserve the geometry of the data, the expected loss in clustering accuracy 
w.r.t. that where all the data are in one place would be small. The key benefit of our approach lies in the fact that the 
generation of codewords is local in the sense that no distributional information from other sites is required. Additionally, 
computation of the codewords can be carried out in parallel on individual nodes thus speeds up the overall computation 
and meanwhile makes the best use of the existing computing resources.  
\\
\\
Our contributions are as follows. Motivated by emerging applications in distributed computing, we propose a new line of research---spectral clustering 
when the data are distributed over multiple sites. The DML-based framework we propose readily facilitates data mining tasks 
such as spectral clustering over distributed data while eliminating the need of big data transmission. Our approach is theoretically 
sound; the loss in the accuracy of spectral clustering, when compared to that in a non-distributed setting, vanishes when increasing 
the size of data transmission. Major computations are {\it naturally} localized and parallelized in the sense that they are carried 
out on the node where the data is stored, which would lead to a major speedup in the overall computation. Additionally, as each 
node is working on part of the full data, this has the potential benefit of divide-and-conquer in the sense that even if we add up the 
computation time at all individual nodes, the total time may still be less than that in a non-distributed setting for big enough data. 
While this work deals with spectral clustering for concreteness and for being an important methodology at its own right, really it 
represents a new paradigm in distributed computing, that is, data mining or inference over distributed data with major computation 
carried out at where the data are located while achieving the effect of using the full data at a ``minimal" communication cost.
\\
\\
The remaining of this paper is organized as follows. In Section~\ref{section:method}, we will describe our framework 
and explain its implementation. This is followed by a discussion of related work in Section~\ref{section:related}. In 
Section~\ref{section:theory}, we present theoretical analysis of our proposed approach. In Section~\ref{section:experiments}, 
we show experimental results. Finally we conclude in Section~\ref{section:conclusion}. 

\section{A framework for spectral clustering on distributed data}
\label{section:method}
The idea underlying our approach is the notion of {\it continuity}. That is, similar data would play a similar role in learning and 
inference, including clustering. Such a notion suggests a class of data transformations called {\it  distortion 
minimizing local (DML) transformation}. The idea of DML transformation is to represent the data by a small set of representative 
points (or codewords). One can think of this as a ``small-loss" data compression, or the representative set as a 
sketch of the full data. Since the representative set resembles the full data, learning based on it is expected to be 
close to that on the full data. Similar idea was explored in \cite{YanHuangJordan2009tech, coresets} and has been 
successfully applied to some computation-intensive algorithms, assuming the data are non-distribted. In all these
previous work, DMLs were mainly introduced to address the computational challenge. 
\begin{figure}[h]
\centering
\begin{center}
\hspace{0cm}
\includegraphics[scale=0.26,clip]{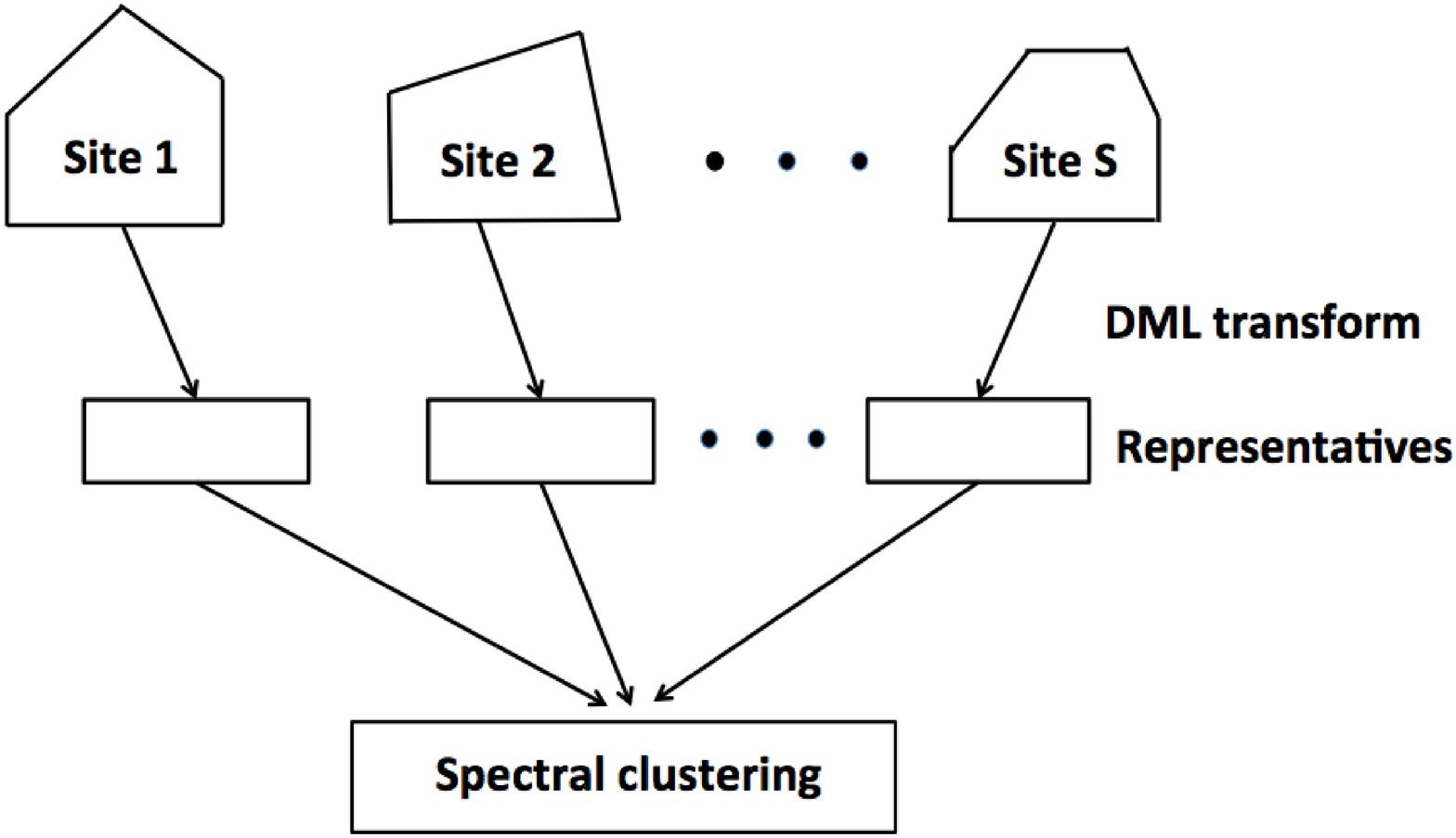}
\end{center}
\caption{\it Architecture of spectral clustering over distributed data. Data at different nodes may be of different
distributions. } 
\label{figure:distInfArch}
\end{figure}
\\
\\
DMLs can be used to enable spectral clustering over {\it distributed} data, that is, the data are not stored in one 
machine but over a number of distributed nodes. Our framework for spectral clustering over distributed data is 
surprisingly simple to implement. It consists of three steps:
 \begin{itemize}
 \item[1)] Apply DML to data at each distributed node
 \item[2)] Collect codewords from all nodes, and carry out spectral clustering on the set 
 of all codewords
 \item[3)] Populate the learned clustering membership by spectral clustering back to each distributed node.
 \end{itemize}
Figure~\ref{figure:distInfArch} is an illustration of our framework. It is clear that algorithms designed 
under this framework would eliminate the need of having to transmit large amount of data among distributed 
nodes---only those codewords need to be transmitted. 
As the codewords are DML-transformed data, data privacy may be ensured since no original data are transmitted. 
Additionally, as spectral clustering is only performed on the set of codewords, the overall computation involved 
will be greatly reduced. Also as the DMLs and the recovering of cluster membership are performed at individual 
nodes, such computation can be done in parallel. We start by a brief introduction to spectral clustering.  
\subsection{Introduction to spectral clustering}
\label{section:introSPC}
Spectral clustering works on an affinity graph over data points $X_1,...,X_N$ and seeks to find a ``minimal" graph cut. 
Depending on the choice of the similarity metric and the objective function to optimize, there are a number of variants, 
including \cite{ShiMalik2000, NgJordan2002}. Our discussion will be 
based on normalized cuts \cite{ShiMalik2000}.
\begin{figure}[h] 
\begin{center}
\includegraphics[scale=0.3,clip]{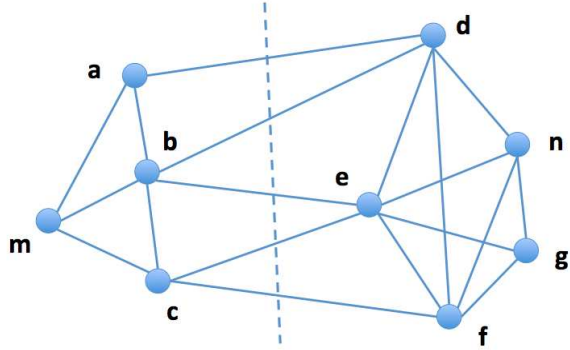}
\caption{\it Illustration of a graph cut. The cut is given by the set $\{ad,bd,be,ce,cf\}$, which partitions the 
vertices of the graph into $V=V_1 \cup V_2=\{a,b,c,m\} \bigcup \{d,e,f,n,g\}$.}
\label{figure:graphCut}
\end{center}
\end{figure} 
An {\it affinity graph} is defined as a weighted graph $\mathcal{G} = (V, \mathcal{E}, A)$
where
$V = \{X_1,...,X_N\}$ is the set of vertices, $\mathcal{E}$ is the edge set, and $A = (a_{ij})_{i,j=1}^N$
is the affinity matrix with $a_{ij}$ encoding the similarity between $X_i$ and $X_j$. 
Figure~\ref{figure:graphCut} is an illustration of graph cut. 
\\
\\
Let $V = (V_1, \ldots, V_K)$ be a partition of $V$. Define the size of the cut between $V_1$ and $V_2$ by
\begin{equation*}
\mathcal{W}(V_1, V_2) = \sum_{i \in V_1, j \in V_2} a_{ij}
~\mbox{for}~V_1, V_2 \subseteq V.
\end{equation*}
Normalized cuts seeks to find the minimal (normalized) graph cut, or equivalently, solve an 
optimization problem
\begin{equation*}  \arg\min_{V_1,...,V_K \subseteq V} \sum_{j=1}^K \frac{\mathcal{W}(V_j, V) -
\mathcal{W}(V_j, V_j)} {\mathcal{W}(V_j, V)}.
\end{equation*}
The above is an integer programming problem thus intractable, a relaxation to real values leads 
to an eigenvalue problem for the Laplacian matrix\footnote{In the following, we 
will omit the subscript $A$ when no confusion is caused.}
\begin{equation}
\label{eq:defLaplacian}
\mathcal{L}_A=D^{-\frac{1}{2}}(D-A)D^{-\frac{1}{2}},
\end{equation}
where $D=diag(d_{1},...,d_{N})$ is the degree matrix with $d_i=\sum_{j=1}^N a_{ij},i=1,...,N$. 
Normalized cuts look for the second smallest eigenvector of $\mathcal{L}_A$, and round its 
components to produce a bipartition of the graph. Similar procedure is applied to each 
of the bipartitions recursively until reaching the number of predefined clusters.
\subsection{Distortion minimizing local transformation}
\label{section:dmlt}
A key property that makes DMLs applicable to distributed data is being {\it local}. That is, such a data 
transformation can be done locally, without having to see the full data. Thus DML 
can be applied at individual distributed nodes separately. If one can pool together all 
those codewords, then overall inference or data mining can be easily carried out. Thus, as 
long as the local data transformations are fine enough, a large class of inference or data mining 
tools will be able to yield result as good as {\it using the full data}. 
\\
\\
As we are dealing with big volumes of data, a natural requirement for a DML is its computational 
efficiency while incurring very ``little" loss in information. We will briefly describe two concrete 
implementations of the DML transformation, one by {\it $K$-means clustering} and the other by 
{\it random projection trees (rpTrees)}, both proposed in \cite{YanHuangJordan2009tech}.    
\subsubsection{K-means clustering}
$K$-means clustering was developed by S. Lloyd in 1957 \cite{lloyd1982}, and 
remains one of the simplest yet popular clustering algorithms. The goal of $K$-means clustering is to 
split data into $K$ partitions (clusters) and assign each point to the ``nearest" cluster (mean). A cluster 
mean is the center of mass of all points in a cluster; it is also called {\it cluster centroid or codewords}. 
The algorithm starts with a set of randomly selected cluster 
centers, then alternates between two steps: 1) assign all points to their nearest cluster centers; 
2) recalculate the new cluster centers. The algorithm stops when the change is small according to 
a cluster quality measure. For a more details about K-means clustering, please refer to the appendix
and \cite{hartiganWong1979, lloyd1982}. 
\\
\\
When DML is implemented by $K$-means clustering, each of the distributed nodes performs $K$-means 
clustering separately. Note that, the number of clusters, $K$, may be different for a different node; the only 
requirement is that $K$ is reasonably large.
The set of representative points are taken as the cluster centers, or average of points in the same cluter 
by an appropriate metric. Empirically, the computation of $K$-means clustering scales {\it linearly} with the 
number of data points, thus it is suitable for big data.    
\subsubsection{Random projection trees}
\begin{figure}[h]
\hspace{-0.3cm}
\begin{center}
\includegraphics[scale=0.25,clip]{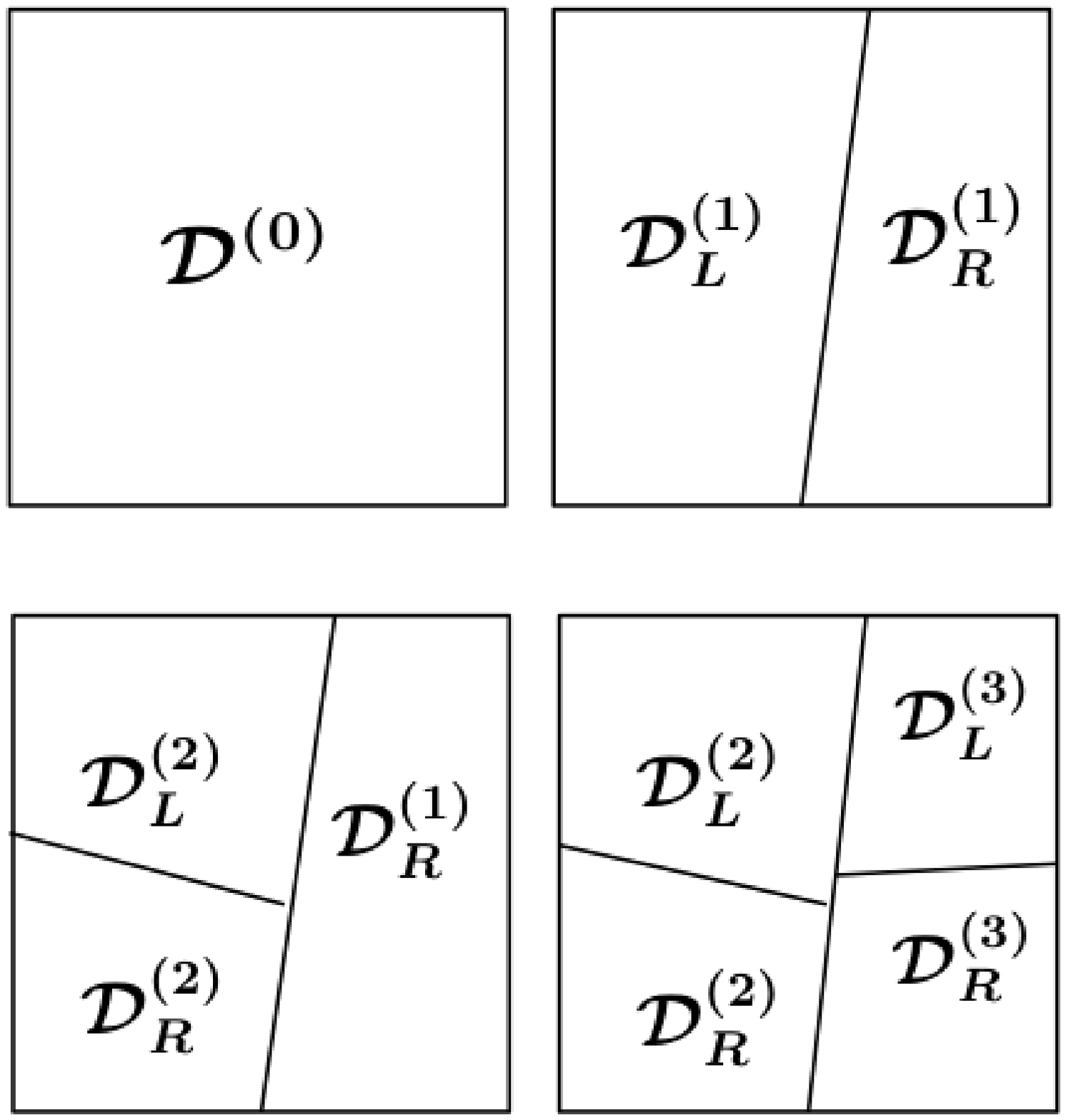}
\includegraphics[scale=0.25,clip]{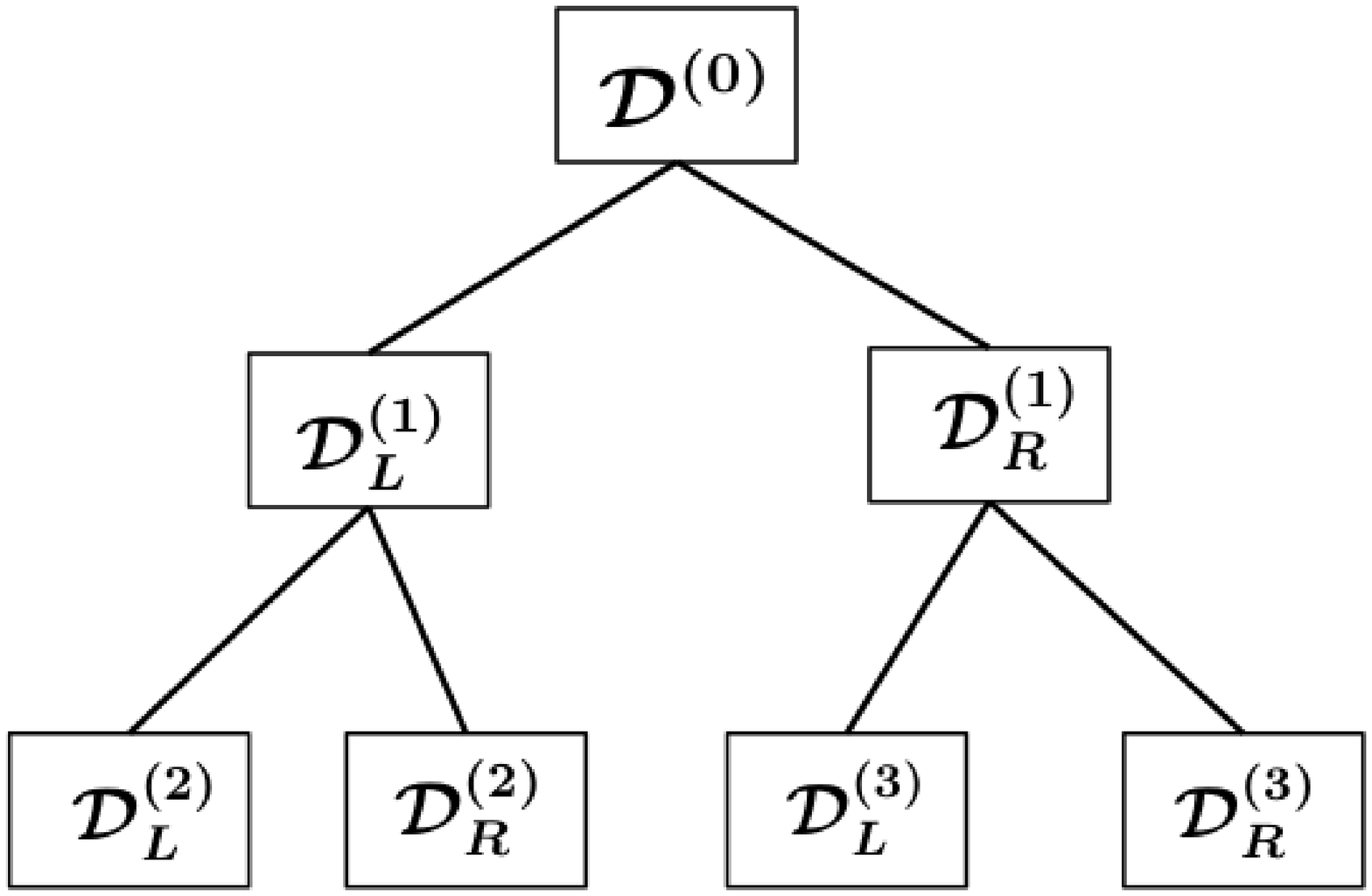}
\end{center}
\caption{\it Illustration of space partition and random projection trees (figure taken from \cite{deepTacoma2019}). 
The superscripts indicate the order of tree node split. One starts with the root node, $\mathcal{D}^{(0)}$, which 
corresponds to all the data. After the first split, $\mathcal{D}^{(0)}$ is partitioned into its two child nodes, 
$\{\mathcal{D}_{L}^{(1)}, \mathcal{D}_{R}^{(1)}\}$. The second split partitions the left child node,
$\mathcal{D}_{L}^{(1)}$, into its two child nodes, $\{\mathcal{D}_{L}^{(2)}, \mathcal{D}_{R}^{(2)}\}$. The 
third cut splits $\mathcal{D}_{R}^{(1)}$ into two new child nodes, $\{\mathcal{D}_{L}^{(3)}, \mathcal{D}_{R}^{(3)}\}$. 
This process continues until a stopping criterion is met.} 
\label{figure:rpTree}
\end{figure}
\noindent
To see how the K-D tree \cite{kdTree} can be used for a DML transformation, we first describe how a K-D tree grows. 
Let the collection of all data at a distributed site correspond to the root node, $\mathcal{D}^{(0)}$, of a tree. Now one 
variable, say $v$ (index of a variable), is selected to split the root node; the root node is split into two child nodes, 
$\mathcal{D}_L^{(1)}$ and $\mathcal{D}_R^{(1)}$, according to whether a data points has its $v$-th coordinate smaller 
or larger than a cutoff point. For each of $\mathcal{D}_L^{(1)}$ and $\mathcal{D}_R^{(1)}$, we will follow a similar procedure
recursively. 
This process continues until some stopping criterion is met. Figure~\ref{figure:rpTree} illustrates the construction of a K-D tree. 
As the tree construction can be viewed as generating a recursive space partition \cite{kdTree,DasguptaFreund2008,YanDavis2018} 
of the data, data points falling in the same leaf node would be similar (if the leaf node is ``small") and their average can be used as the representative. If the 
underlying data has a very high dimension, then a randomized variant, rpTrees \cite{DasguptaFreund2008}, can be used which 
would adapt to the geometry of the underlying data and effectively overcome the {\it curse of dimensionality}. In this work, we use 
rpTrees implemented in \cite{rpForests2018c}.  
\subsection{Algorithmic description}
\label{section:algorithm}
Now we can briefly describe how to adopt spectral clustering for distributed data. Assume that there are $S$ distributed 
sites. Apply DML (K-means clustering or rpTrees) at each site individually. Let $Y_i^{(s)}, i=1, 2, ..., n_s$ 
be the group centroids of data at site $s=1,2,...,S$. A group is either data in the same cluster if K-means clustering is 
used, or all points in the same leaf node of rpTrees. The centroids are center of mass of all points in the same group. 
Spectral clustering is performed on the set of group centroids (representative points) collected from all the $S$ sites.
An algorithmic description is given in Algorithm~\ref{algorithm:distSpect}.  
\begin{algorithm}
\caption{~~Spectral clustering for distributed data} 
\label{algorithm:distSpect}
\begin{algorithmic}[1]
\STATE $\mathcal{D}_r \gets \emptyset$;
\FOR {each site $s \in \{1,...,S\}$} 
\STATE Apply DML to data at site s; 
\STATE Let $Y_i^{(s)}, i=1, 2, ..., n_s$ be the group centroids;  
\STATE Let $W_i^{(s)}, i=1, 2, ..., n_s$ be the group sizes; 
\STATE $\mathbb{Y}_s \gets \{Y_i^{(s)}:  i=1, 2, ..., n_s \}$;
\ENDFOR 
\STATE Collect group centroids from all sites $\mathcal{D}_r \gets \cup_{s=1}^S \mathbb{Y}_s$;
\STATE Spectral clustering on $\mathcal{D}_r$;
\STATE Populate cluster membership to all $S$ sites; 
\end{algorithmic}
\end{algorithm} 
If the DML transformation is linear, which is the case if it is implemented by K-means clustering or rpTrees,
then {\it the overall computational complexity is easily seen to be linear in the total number of points in 
the distributed data}. Indeed for large scale distributed computation, that is an implicit requirement. 
\section{Related work}
\label{section:related}
There has been an explosive growth of interests in distributed computing in the last decades. One driving force behind is the 
prevalence of low-cost clustered computers and storage systems \cite{SortingWorkstation1997,FoxGribble1997}, which makes it 
feasible to interconnect hundreds or thousands of clustered computers. Numerous systems and computing platforms have 
been developed for distributed computing. For example, Google's Bigtable \cite{GoogleFileSys2003,Bigtable2006}, the 
Apache Hadoop/Map-Reduce \cite{MapReduce2004, HDFS2010}, the Spark system \cite{ZahariaChowdhury2010,ZahariaChowdhury2012}, and Amzon's AWS cloud etc. The literature is huge, but mostly on distributed system architecture, computing platforms, or data query tools. 
For a review of recent developments, please see \cite{ChenMaoLiu2014,Dolev:2017TBD} and references therein.
\\
\\
Existing distributed algorithms in the literature are either parallel algorithms, for example, \cite{ChenSongChang2011}, or
use a divide-and-conquer strategy which split the data and distribute the workload to a number of nodes \cite{HefeedaGao2012}. 
An influential line of work is {\it Bag of Little Bootstrap} \cite{BagLittleBootstrap}. This work aims at 
computing a big data version of Bootstrap \cite{Efron1979}, a fundamental tool in statistical inference; the idea is to take many 
very ``thin" subsamples, distribute the computing on each subsample to a node and then aggregate results from those individual 
subsamples. Recently, \cite{BatteyFan2015} considered the general distributed estimation and inference 
in the Divide-and-Conquer paradigm, and obtained optimal partitions for data. Chen and Xie \cite{ChenXie2014} studied 
penalized regression and model selection consistency when the data is too big to fit into the memory of a single machine by 
working on a subsample of the data and then aggregating the resulting models. Singh et al \cite{Singh:2017TBD} designed 
DiP-SVM, a distribution preserving kernel support vector machine where the first and second order statistics of the data are 
retained in each of the data partitions, and run spectral clustering on each data partition at an individual nodes with results 
aggregated. One thing common about these work is that all assume one has the entire data before delegating the workload 
to individual nodes or pushing the data to temporary storage and then aggregate the individual results; the data are distributed 
or split mainly for improving computational efficiency or solving the memory shortage problem. 
\\
\\
Many works have been proposed to improve the efficiency of spectral clustering. Chen and Cai \cite{Chen2011LargeSS} 
proposed a landmark-based method for spectral clustering by selecting representative data points as a linear combination of the 
original data. Zhang and his co-authors \cite{Zhang2016SamplingFN} proposed an incremental sampling approach, i.e., the 
landmark points are selected one at a time adaptively based on the existing landmark points. Liu et al. \cite{Liu2017FastCS} 
proposed a fast constrained spectral clustering algorithm via landmark-based graph construction and then reduce the data size 
by random sampling after spectral embedding. Paiva \cite{Paiva2017InformationtheoreticDS} proposed to select a representative
subset of the training sample with an information-theoretic strategy. Lin et al \cite{Lin2017ASA} designed a scalable co-association 
cluster ensemble framework using a compressed version of co-association matrix formed by selecting representative points of the 
original data. 
\\
\\
Several recent work aims at reducing the data volume by deep learning. Aledhari et al \cite{Mohammed:2017TBD} proposed 
a deep learning based method to minimize large genomic DNA dataset for transmission through the internet. Banijamali et al 
\cite{Banijamali2017FastSC} integrated the recent deep auto-encoder technique into landmark-based spectral clustering. 
\section{Analysis of the algorithm}
\label{section:theory}
In our distributed framework, each node $s$ individually applies DML to generate codewords $\mathbb{Y}_s=\{Y_i^{(s)}:  i=1, 2, ..., n_s \}$, 
to be sent to a central node for spectral clustering. The results from spectral clustering are returned to node $s$ for the recovery 
of cluster membership for all points at node $s=1,...,S$. A crucial question is, does this approach work? Since each site performs 
DML locally and no site is using distributional information from others. How much additional error will be incurred under our 
framework, or will such an error vanish if the data is big? The goal of our analysis is to shed lights to these questions.
\\
\\
To carry out such an analysis is challenging. We are interested in the clustering error (an end-to-end error), but we only observe 
errors in terms of the local distortion in representing the points, $\mathbb{X}_s =\{X_i^{(s)}:  i=1, 2, ..., N_s \}$ by codeword 
$\mathbb{Y}_s$ for each $s=1,...,S$. We need to establish a connection between the clustering error and the distortion in data 
representation of a distributed (local) nature. 
\begin{figure}[h]
\hspace{-0.4cm}
\begin{center}
\includegraphics[scale=0.62,clip,angle=-90]{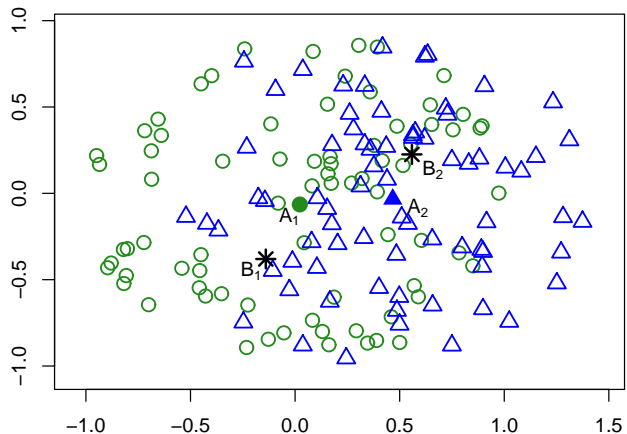}
\end{center}
\abovecaptionskip=-5pt
\caption{\it Illustration when the support of data from two different sites overlaps. Data with the same color indicate the same 
distributed node. The original codewords computed at each node are $A_1$ and $A_2$ (marker by circle and triangle, 
respectively), while the optimal codewords (assuming there are two) for the combined data are $B_1$ and $B_2$ (marked by stars).} 
\label{figure:overlap}
\end{figure}
\\
\\
A second challenge is related to the local `optimality' of our algorithm. That is, each set of codewords $\mathbb{Y}_s$ is an 
`optimal' representation of data, $\mathbb{X}_s$, at a local site, but their union, $\mathbb{Y}=\cup_{s=1}^S \mathbb{Y}_s$, 
is not necessary the optimal set of codewords for the entire set of data $\mathbb{X}=\cup_{s=1}^S \mathbb{X}_s$. One situation 
for which this happens is when there is an overlap among the support of points from different sites. As illustrated in Figure~\ref{figure:overlap}, 
the original codes $A_1, A_2$ are no longer the optimal codewords for the combined data; rather the new optimal codewords 
become $B_1$ and $B_2$. Now the question is, {\it will local `optimality' at all individual sites be sufficient or if additionally 
constraints may be required to make our approach work?}  
\\
\\
One crucial insight to our analysis is, what we really need is the convergence of the global distortion to zero when the size of 
data increases, not necessary the optimality of the global distortion. So we would proceed {\it as if} all the data are from one 
node, and only plugs in analysis of local DMLs when necessary.  
Our analysis hinges on an important result obtained in \cite{YanHuangJordan2009tech} which establishes 
the connection between the end-to-end error and the perturbations to the Laplacian matrix due to data distortion. 
\begin{lemma}[\cite{YanHuangJordan2009tech}] 
\label{lemma:errEigenLaplacian}
Under assumptions $\mathbb{A}_{1-3}$, the
mis-clustering rate $\rho$ of a spectral bi-partitioning algorithm
on perturbed data satisfies
\begin{equation*}
\label{eqn:mis-clustering-rate} 
\rho \leq
\|\bm{\tilde{v}}_2-\bm{v}_2\|^2 \leq ||\tilde{\mathcal{L}}-\mathcal{L}||_F^2,
\end{equation*}
where $||.||_F$ indicates the Frobenius norm \cite{GolubVanLoan1989}.
\end{lemma}
\noindent
By Lemma~\ref{lemma:errEigenLaplacian}, in order to bound the additional clustering error resulting from the
distributed nature of the data, we just need to bound the distortion of the Laplacian matrix due to a compressed 
data representation by codewords from a number of distributed sites. 
\\
\\
We will proceed in two steps. First, we conduct a perturbation analysis of the Laplacian matrix. Then, we will 
adopt results from local DMLs to the perturbation results. Note that in the perturbation analysis,
we treat all the data as if they are all from one site. We follow notations in \cite{YanHuangJordan2009tech}.
\\
\\
Our theoretical model is the following two-component Gaussian mixture
\begin{equation}
\label{eq:gm}
G =(1-\pi)\cdot G_1 + \pi \cdot G_2,
\end{equation}
where $\pi \in \{0,1\}$ with $\mathbb{P}(\pi=1)=\eta$. The choice of a two-component Gaussian
mixture is mainly for simplicity; it should be clear that our analysis applies to any finite Gaussian
mixtures. We treat data perturbation as adding a noise component $\epsilon$ to $X$:
\begin{equation}
\label{eq:noise}
\tilde{X}=X+\epsilon,
\end{equation}
and we denote the distribution of $\tilde{X}$ by $\tilde{G}$.
We assume $\epsilon$ is symmetric about $0$ with bounded support,
and let $\epsilon$ has a standard deviation $\sigma_{\epsilon}$ that
is small compared to $\sigma$, the standard deviation for the distribution of $X$.
\subsection{Perturbation analysis}
Our perturbation analysis relies on one more result established in \cite{YanHuangJordan2009tech} which is stated 
as Lemma~\ref{lemma:laplacianIneq}.
\begin{lemma}[\cite{YanHuangJordan2009tech}] 
\label{lemma:laplacianIneq}
Let $\mathcal{L}$ and $\tilde{\mathcal{L}}$ be the Laplacian matrix corresponding to the original similarity
matrix and that after perturbation. Then
\begin{multline}
\label{eq:pertlap_F} 
||\tilde{\mathcal{L}}-\mathcal{L}||_F \leq ||D^{-\frac{1}{2}}E D^{-\frac{1}{2}}||_F \\ +(1+o(1))||\Delta
D^{-\frac{3}{2}} AD^{-\frac{1}{2}}||_F. 
\end{multline}
\end{lemma}
\noindent
We state without proof some elementary results.
\begin{lemma}
\label{lemma:easyInq}
Let $a, b \in \mathbb{R}$. Then the following holds
\begin{equation*}
(a-b)^2 \leq 2(a^2 +b^2), ~ (a-b)^4 \leq 8(a^4 +b^4).
\end{equation*}
\end{lemma}
\noindent
The main result of our perturbation analysis is stated as Theorem~\ref{thm:pertBound}.
\begin{theorem}
\label{thm:pertBound}
Suppose $X_1,...,X_N \in \mathbb{R}^d$ are generated i.i.d. according to \eqref{eq:gm} such that
$\inf_{1\leq i \leq N} d_i /N>\delta_0$ holds in probability for some constant $\delta_0>0$. Further, 
the number of data points, $N_s$, at site $s$ carries a substantial fraction of the total number of data 
points $N$ in the sense that $\lim_{N \rightarrow \infty} N_s/N=\gamma_s \in (0,1)$ for $s=1,...,S$. 
Assume the data perturbation $\epsilon$ is symmetric about $0$ with bounded support. Further 
assume $||\Delta D^{-1}||_2=o(1)$. Then
\begin{eqnarray*}
||\tilde{\mathcal{L}}-\mathcal{L}||_F^2 \leq_{p}
C \sum_{s=1}^S \gamma_s \sigma_{s}^2 + C' \sum_{s=1}^S \gamma_s \sigma_{s}^4
\end{eqnarray*}
for some constants $C$ and $C'$, as $N \rightarrow \infty$.
\end{theorem}
\begin{proof}
Our proof essentially follows that of Theorem 5 and Theorem 6 in \cite{YanHuangJordan2009tech} except during the final steps 
of Lemma 13 and Lemma 14. Instead of using theory of U-statistics \cite{Hoeffding1961} which would give a sharper bound, 
we use an upper bound, i.e., Lemma~\ref{lemma:easyInq}, which allows us to group the perturbation error of individual data points 
by the distributed site they belong to. By adopting the proof of Lemma 13, we have
\begin{eqnarray*}
&& ||D^{-\frac{1}{2}} E D^{-\frac{1}{2}}||_F^2\\
&\leq& \frac{2C}{\delta_0^2N^2}\sum_{i=1}^N \sum_{j=1}^N (\epsilon_i
- \epsilon_j)^2 + \frac{2}{\delta_0^2N^2}\sum_{i=1}^N \sum_{j=1}^N
R_{max}^2(\epsilon_i-\epsilon_j)^4 \\
&\leq& \frac{4C}{\delta_0^2N^2}\sum_{i=1}^N \sum_{j=1}^N (\epsilon_i^2
+ \epsilon_j^2) + \frac{16}{\delta_0^2N^2}\sum_{i=1}^N \sum_{j=1}^N
R_{max}^2(\epsilon_i^4 + \epsilon_j^4) \\
&=& \frac{8C}{\delta_0^2N}\sum_{i=1}^N \epsilon_i^2 +
 + \frac{32}{\delta_0^2N}\sum_{i=1}^N
R_{max}^2 \epsilon_i^4 \\
&\overrightarrow{a.s.}& C_1 \sum_{s=1}^S \gamma_s \sigma_{s}^2 + C_2 \sum_{s=1}^S \gamma_s \sigma_{s}^4
\end{eqnarray*}
as $N \rightarrow \infty$. In the above $\epsilon_i$ is the perturbation to observation $X_i$ for $i=1,...,N$. 
Similarly in the proof of Lemma 14, we have 
\begin{eqnarray*}
&& ||\Delta D^{-\frac{3}{2}}AD^{-\frac{1}{2}} ||_F^2\\
&\leq& \frac{2}{\delta_0^4N^2}\sum_{i=1}^N\sum_{k=1}^N \left[C(\epsilon_i-\epsilon_k)^2+R_{max}^2(\epsilon_i-\epsilon_k)^4\right]\\
&\overrightarrow{a.s.}& C_3 \sum_{s=1}^S \gamma_s \sigma_{s}^2 + C_4 \sum_{s=1}^S \gamma_s \sigma_{s}^4
\end{eqnarray*}
as $N \rightarrow \infty$.
Combining the above two inequalities and then apply Lemma~\ref{lemma:laplacianIneq}, we have proved the theorem.
\end{proof}
\subsection{Quantization errors by K-means clustering}
Existing work from vector quantization~\cite{QuantError,Quantization}
allows us to characterize the amount of distortion when the
representative set is computed by $K$-means clustering.
\\
\\
Let a quantizer $q$ be defined as $q: \mathbb{R}^d \mapsto
\{y_1,\ldots,y_n\}$ for $y_i \in \mathbb{R}^d$.  For $X$ generated
from a random source in $\mathbb{R}^d$, let the distortion of $q$ be
defined as: $\mathcal{D}(q)=\mathbb{E}||X-q(X)||^s$, which is the
mean square error for $s=2$.  Let $R(q)=\log_2 n$ denote the rate of
the quantization code.
Define the distortion-rate function $\delta(R)$ as
\begin{equation*}
\delta(R)=\inf_{q:~R(q) \leq R} \mathcal{D}(q).
\end{equation*}
Then $\delta(R)$ can be characterized in terms of the source density
$f(\cdot)$, and constants $d, s$ by the following theorem.
\begin{theorem}[\cite{QuantError,Quantization}]
\label{thm:quantization} Let $f(x)$ be the $d$-dimensional probability
density. Then, for large rates $R$, the distortion-rate function of
fixed-rate quantization can be characterized as:
\begin{equation*}
\delta_d(R) \cong b_{s,d}\cdot||f||_{d/(d+s)} \cdot n^{-s/d},
\end{equation*}
where $\cong$ means the ratio of the two quantities tends to 1,
$b_{s,d}$ is a constant depending on $s$ and $d$, and
\begin{equation*}
||f||_{d/(d+s)}=\left(\int f^{d/(d+s)}(x)dx\right)^{(d+s)/d}.
\end{equation*}
\end{theorem}
\noindent
Now we can apply Theorem~\ref{thm:quantization} and Pollard's strong consistency of K-means clustering \cite{Pollard1981} 
to individual terms in Theorem~\ref{thm:pertBound}, and establish a performance bound for our approach.
\begin{theorem}
\label{thm:kmBound}
Assume the data distributions in a S-site distributed environment have density $f_1,...,f_S$, respectively. 
Let $k_1,...,k_S$ be the number of codewords at the $S$ distributed sites. Let $k = min(k_1,...,k_S)$.
Then the additional clustering error rate $\rho$, as compared to non-distributed setting, under our framework 
can be bounded by
\begin{equation*}
C \cdot \max_{s=1,...,S}  ||f_s||_{d/(d+2)}.k^{-2/d}  + O(k^{-4/d}),
\end{equation*}  
where $C$ is a constant determined by the number of clusters, the variance of the original data, the bandwidth
of the Gaussian kernel and the eigengap of Laplacian matrix (or minimal eigengap of the Laplacian of all
affinity matrices used in normalized cuts).
\end{theorem}
\section{Experiments}
\label{section:experiments}
In this section, we report our experimental results. This includes simulation results on synthetic data, and on data from the UC Irvine
Machine Learning Repository \cite{UCI}. We will compare the performance of distributed vs {\it non-distributed} (where all the data are 
assumed to be in one place). The spectral clustering algorithm used is normalized cuts \cite{ShiMalik2000}, and the Gaussian kernel 
is used in computing the affinity (or Gram) matrix with the bandwidth chosen via a cross-validatory search in the range (0, 200] (with 
step size 0.01 within (0,1], and 0.1 over (1,200]) for each data set. All algorithms are implemented in {\it R programming language}, and 
the {\it kmeans()} function in R is used for which details are similar as those documented in \cite{YanHuangJordan2009tech}. 
\\
\\
The metric for clustering performance is {\it clustering accuracy}, which counts the fraction of
labels given by a clustering algorithm that agree with the true labels (or labels come with the dataset). 
Let $\{1,...,K\}$ denote the set of class labels, and $h(.)$ and $\hat{h}(.)$ are the
true label and the label given by the clustering algorithm, respectively. The clustering accuracy is defined as \cite{YanHuangJordan2009tech}
\begin{equation}
\label{eq:clusterAccuracy}
\max_{\tau \in \Pi} \left\{\frac{1}{N}\sum_{i=1}^N
\mathbb{I}\{\tau(h(x_i))=\hat{h}(x_i)\}\right\},
\end{equation}
where $\mathbb{I}$ is the indicator function and $\Pi$ is the
set of all permutations on the class labels $\{1,...,K\}$. While there are dozens of performance metrics around for
clustering, the clustering accuracy is a good choice for the evaluation of clustering algorithms. This is because 
the label (or cluster membership) of individual data points is the ultimate goal of clustering, while other clustering metrics 
are often a surrogate of the cluster membership, and they are used in practice mainly due to the lack of labels. 
For the evaluation of clustering algorithms, we do have the freedom of using those datasets coming with a label. Indeed the 
clustering accuracy is commonly used for the evaluation of clustering; see, for example, \cite{XingNgJordanRussell2002,BachJordan2006,YanHuangJordan2009tech}. 
\\
\\
We also consider the {\it time for computation}. The elapsed time is used. It is counted from the moment when the data are loaded 
into the memory (R running time environment), till the time when we have obtained the cluster label for all data points. We assume 
all the distributed nodes run independently, so the longest computation time among all the sites is used (instead of adding them up). 
As we do not have multiple computers for the experiments, we do not explore the communication time for transmitting representative 
points and the clustering results. Indeed, for the dataset used in our experiments, such time could be ignored when compared to the 
computation time, as the number of representative points are all less than 2000. The time reported in this work are 
produced on a {\it MacBook Air} laptop computer, with 1.7GHz Intel Core i7 processor and 8G memory. 
\subsection{Synthetic data}
\label{section:expSynthetic}
To appreciate how those representative points look like, we first consider a toy example. 
\begin{figure}[h]

\hspace{-0.6cm}
\begin{center}
\includegraphics[scale=0.32,clip, angle=-90]{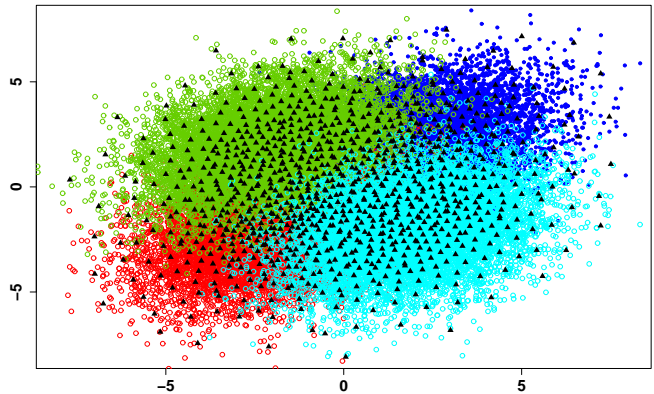}
\end{center}
\caption{\it A 4-component Gaussian mixture. Different clusters are indicated by colors. The triangle spots are representative 
points. Site 1 corresponds to red and blue colors, and Site 2 green and cyan.} 
\label{figure:gmix}
\end{figure}
The data is generated by a 4-component Gaussian mixture $\frac{1}{4}\sum_{i=1}^4 \mathcal{N}(\mu_i, \Sigma) $ 
with $\mu_1=(2,2)$, $\mu_2=(-2,-2)$, $\mu_3=(-2,2)$ and $\mu_4=(2,-2)$, and 
the covariance matrix given by \[\Sigma=\begin{bmatrix} 
3 & 1 \\
1 & 3 \end{bmatrix}.\]
The Gaussian mixture is commonly used as a generating model for clusters in the literature \cite{Dasgupta1999,MclachlanPeel2000,LuxburgBelkin2008, TACOMA, CF}, due to its simplicity (as the mixture components are Gaussians) and also because it is amenable for theoretical analysis. 
Indeed, it is used as the theoretical model in our analysis (c.f. Section~\ref{section:theory}). Note that here we are not arguing that spectral 
clustering is superior to other clustering methods, rather our claim was spectral clustering could work in a distributed setting that we consider. 
\\
\\
Figure~\ref{figure:gmix} is a scatter plot of the Gaussian mixture. The four mixture components are marked by {\it blue}, {\it red},
{\it green}, and {\it cyan}, respectively. Here in generating representative points, we run $K$-means clustering on the two sites 
defined by $\mathcal{N}(\mu_1, \Sigma) + \mathcal{N}(\mu_2, \Sigma) $ and $\mathcal{N}(\mu_3, \Sigma)  + \mathcal{N}(\mu_4, \Sigma) $, respectively.
It can be seen that those representative points serve as a ``good" condensed representation of the original data points. 
\\
\\
We conduct experiments on a 4-component Gaussian mixture on $\mathbb{R}^{10}$
\begin{equation}
\frac{1}{4}\sum_{i=1}^4 \mathcal{N}(\mu_i, \Sigma) 
\label{equation:gm}
\end{equation}
with the center of the four components being 
\begin{eqnarray*}
&& \mu_1=(2.5,0,...,0), ~\mu_2=(0,2.5,0,...,0), \\
&& \mu_3=(0,0,2.5,0,...,0), ~\mu_4=(0,0,0,2.5,0,...,0),
\end{eqnarray*}
and the covariance 
matrix $\Sigma$ defined by 
\begin{equation*}
\Sigma_{i,j}=\rho^{\vert i-j \vert}, ~~\mbox{for}~ \rho=0.1, 0.3, 0.6.
\end{equation*}
For simplicity, we only consider the case with two sites in a distributed environment. It should be easy to extend to the general case. 
Let $\mathcal{C}_i, i=1, 2, 3, 4$, denote the data from the four mixture components, respectively. To simulate how the data may
look like in a distributed environment, we create the following three scenarios:
\begin{itemize}
\item[$D_1$:] Site 1 has $\mathcal{C}_1 + \mathcal{C}_2$, Site 2 has those from $\mathcal{C}_3 + \mathcal{C}_4$ \vspace{0.02in}
\item[$D_2$:] Site 1 has $\frac{1}{2}\mathcal{C}_1 + \mathcal{C}_2+ \frac{1}{2}\mathcal{C}_3$, Site 2 has $ \frac{1}{2}\mathcal{C}_1 + \frac{1}{2}\mathcal{C}_3 + \mathcal{C}_4$ \vspace{0.02in}
\item[$D_3$:] Site 1 and 2 each has a randomly selected half of the data points
\end{itemize}  
where $\frac{1}{2} \mathcal{C}_1$ means that a distributed site contains half of the data points from component  $\mathcal{C}_1$, 
and so on. Note these scenarios are {\it not} the different ways that we split and distribute the data for fast computation rather each 
should be viewed as one type of distributed settings: $D_1$ for which data at different sites have roughly disjoint supports, $D_2$
for which data at different sites have some overlap in terms of supports, and in $D_3$ individual sites have similar data distribution.
40000 data points are generated from the Gaussian mixture \eqref{equation:gm}, and the number of representative points 
is $1000$ (i.e., the data compression ratio is 40:1). The number of data points at Site 1 and 2, and also the number of representative 
points can all be calculated by the site specification and the data compression ratio accordingly. Note that here the number of clusters 
in K-means clustering (or the number of leaf nodes in rpTrees) is no longer as important as that in usual clustering, as long as the resulting
groups are sufficiently fine (i.e., data points in the same group are ``similar"); since here K-means clustering
(or rpTrees) is used to group the data for efficient distributed computation. The same discussion applies to UC Irvine datasets as well. 
\begin{figure}[h]
\hspace{-0.2cm}
\begin{center}
\includegraphics[scale=0.54,clip,angle=-90]{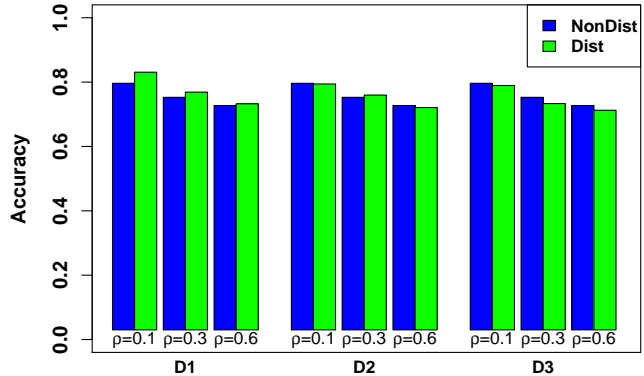}
\end{center}
\abovecaptionskip=-20pt
\caption{\it Clustering accuracy by spectral clustering on a 4-component Gaussian mixture with different $\rho$'s
when K-means clustering is used as DML. $D_1,D_2$ and $D_3$ are 3 simulation scenarios, and 
there are two distributed sites in all simulations.} 
\label{figure:barSynKmeans}
\end{figure}
\begin{figure}[h]
\vspace{-0.1in}
\hspace{-0.3cm}
\begin{center}
\includegraphics[scale=0.54,clip,angle=-90]{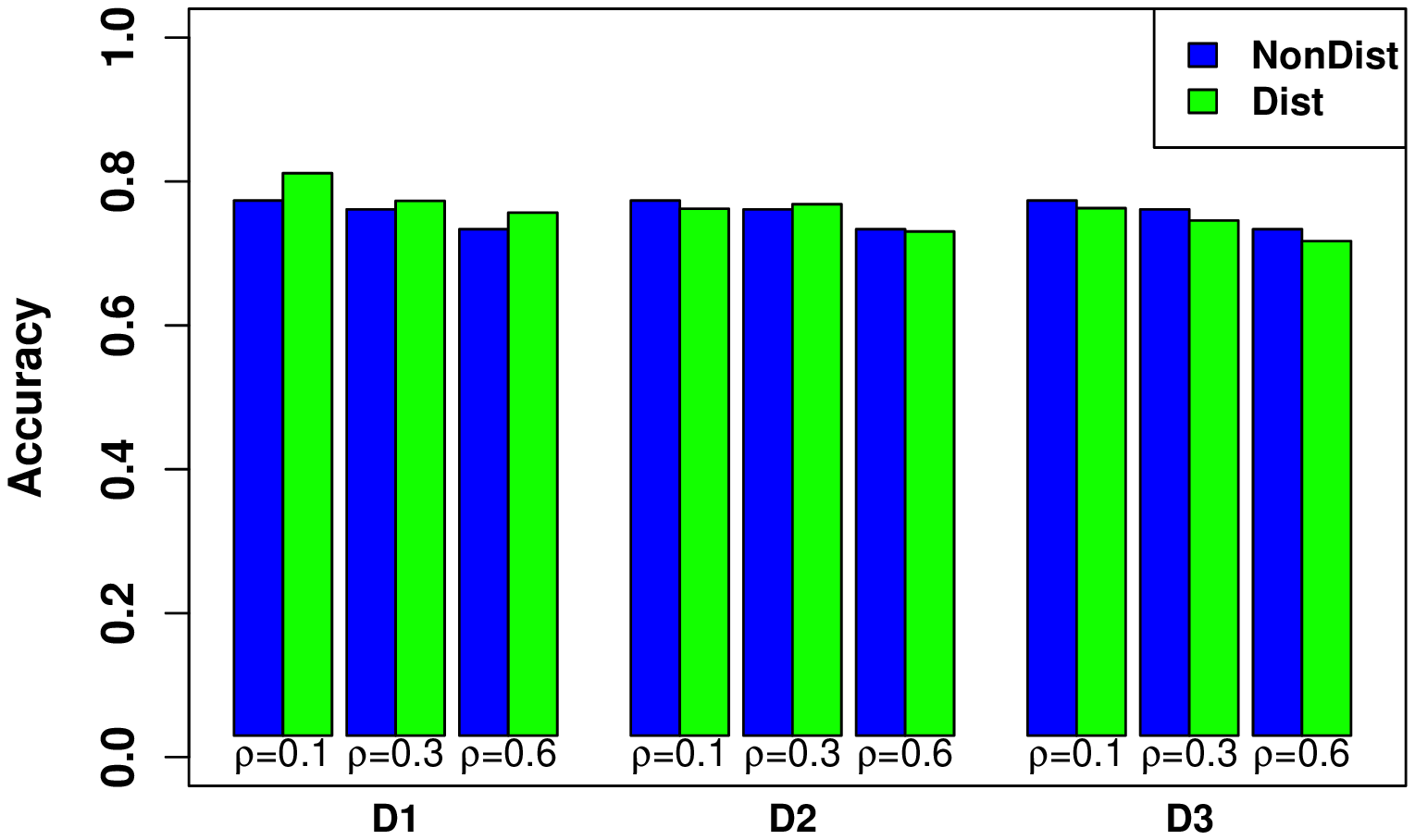}
\end{center}
\abovecaptionskip=-20pt
\caption{\it Clustering accuracy by spectral clustering on a 4-component Gaussian mixture with different $\rho$'s
when rpTrees is used as DML. $D_1,D_2$ and $D_3$ are 3 simulation scenarios, and 
there are two distributed sites in all simulations.} 
\label{figure:barSynRPtrees}
\end{figure}
\\
\\
Figure~\ref{figure:barSynKmeans} and Figure~\ref{figure:barSynRPtrees} show the clustering accuracy of spectral 
clustering under various distributed scenarios as compared to non-distributed when K-means clustering and rpTrees are used as 
DML, respectively. Here the true label of each data point is taken as its mixture component ID. The data compression 
ratio for K-means clustering is set to be 40:1, which is the ratio of the size of the data to the number of clusters in DML; 
the maximum size of the leaf nodes is 40 for rpTrees, to match approximately the data compression ratio in K-means 
clustering. In all cases, the accuracy computed for distributed data is close to non-distributed. The reason that the accuracy 
under $D_1$ may be higher than others (including the non-distributed) may be attributed to the fact the data are {\it less mixed} 
than in other settings thus implicitly achieving a {\it regularization effect} where data are grouped within the same classes 
or the part of data that are less mixed than others (one can roughly view this as additional constraints on clustering). 
\begin{table}[htb]
\begin{center}
\setlength{\extrarowheight}{2pt}
\begin{tabular}{c|crc}
    \hline
\textbf{Data set}     & \textbf{\# Features}  &\textbf{\# instances}  &\textbf{\# classes} \\
    \hline
Connect-4   & 42       &67,557             &3\\
SkinSeg      & 3         &245,057	       &2\\
USCI           & 37       &285,779            &2\\
Cover Type	&54		&568,772				&5\\
HT Sensor   &11        &928,991             &3\\
Poker Hand  & 10       &1,000,000         &3\\
Gas Sensor  &18		& 8,386,765		 &2\\
HEPMASS    &28     &10,500,000         &2\\
\hline
\end{tabular}
\end{center}
\caption{\it A summary of UC Irvine data used in the experiments. } \label{tbl:data}
\end{table}
\subsection{UC Irvine data}
\label{section:expUCI}
\begin{table}[htb]
\begin{center}
\setlength{\extrarowheight}{2pt}
\begin{tabular}{@{}c@{}|c|@{}c@{}|@{}c@{}|@{}c@{}|c|c}
    \hline
\textbf{Data set}&\multicolumn{2}{c}{$\bm{D_1}$}&\multicolumn{2}{|c}{$\bm{D_2}$}&\multicolumn{2}{|c}{$\bm{D_3}$} \\[2pt]
    \hline
USCI&\multirow{4}{*}{$\mathcal{C}_1$}&\multirow{4}{*}{$\mathcal{C}_2$}&~\multirow{4}{*}{0.7$\mathcal{C}_1$+0.3$\mathcal{C}_2$}~ &~\multirow{4}{*}{0.3$\mathcal{C}_1$+0.7$\mathcal{C}_2$}~ &\multirow{4}{*}{50\%} & \multirow{4}{*}{50\%}\\\cline{1-1}
SkinSeg        &     & & &    & &  \\\cline{1-1}
Gas Sensor  &     & & &    & &  \\\cline{1-1}
HEPMASS   &     & & &    & &\\\hline
Connect-4&\multirow{3}{*}{$\mathcal{C}_1$}&~\multirow{3}{*}{$\mathcal{C}_2$+$\mathcal{C}_3$}~ &\multirow{3}{*}{0.5$\mathcal{C}_1$+$\mathcal{C}_2$}
&\multirow{3}{*}{0.5$\mathcal{C}_1$+$\mathcal{C}_3$} &\multirow{3}{*}{50\%}&\multirow{3}{*}{50\%}\\\cline{1-1}
HT Sensor  &    &    &   &    & &  \\\cline{1-1}
~Poker Hand~  &    &    &   &    & &  \\\hline
\multirow{2}{*}{Cover type}&\multirow{2}{*}{$\mathcal{C}_2$}&~\multirow{2}{*}{$\mathcal{C}_1$+$\mathcal{C}_{3-5}$}~ &\multirow{1}{*}{0.7$\mathcal{C}_1$+0.3$\mathcal{C}_2$}
&\multirow{2}{*}{0.3$\mathcal{C}_1$+0.7$\mathcal{C}_2$} &\multirow{2}{*}{50\%}&\multirow{2}{*}{50\%}\\
&&&\multirow{1}{*}{+$\mathcal{C}_{3-5}$}&&&\\
\hline
\end{tabular}
\end{center}
\caption{\it A summary of simulation settings. Here $D_1, D_2, D_3$ are 3 distributed simulation scenarios 
considered in this work. } \label{tbl:simulationSetting}
\end{table}
The UC Irvine datasets we use include the Connect-4, SkinSeg (Skin segmentation), USCI (US Census Income), 
Cover type, HT Sensor, Poker Hand, Gas Sensor and the HEPMASS data. Table~\ref{tbl:data} gives a summary 
of the datasets. For Connect-4, 
USCI, and Poker Hand data, we follow procedures described in \cite{YanHuangJordan2009tech} 
to preprocess the data. The original {\it USCI} data has 299,285 instances with 41 features. We excluded instances with 
missing values, and also features \#26, \#27, \#28 and \#30, due to too many values. This leaves 285,799 instances on 
$37$ features, with all categorical variables converted to integers. The original {\it Cover Type} data has 581,012 instances. 
We excluded the two small classes (i.e., 4 and 5) for fast evaluation of accuracy (otherwise all 7! permutations need to be 
evaluated, but that is not our focus of the feasibility of distributed computing), and this leaves 568,772 instances; 
we also standardized the first 10 features such that each has a mean 0 and variance 1. The original {\it Poker Hand} data is 
highly unbalanced, with $6$ small classes containing less than $1\%$ of the data. Merging small classes gives $3$ 
final classes with a class distribution of $50.12\%$, $42.25\%$ and $7.63\%$, respectively. The {\it Gas Sensor} data 
consists of two different gas mixtures: Ethylene mixed with CO, and Ethylene mixed with Methane. The data corresponding
to these two gas mixtures form the two classes in the data. The Connect-4, the USCI,  and the Gas Sensor data are 
normalized such that all features have mean $0$ and standard deviation $1$. 
\\
\\
It is worthwhile to remark that, though the datasets we use may be smaller than some large data at industry-scale, the Gas 
Sensor and the HEPMASS are among the largest datasets currently available in the UC Irvine data repository. Moreover, in industry
often the raw data used may be huge but, after preprocessing and feature engineering, the data used for data mining or 
model building are substantially smaller. Additionally, the UC Irvine datasets used in our experiments have a wide variety of 
sizes, ranging from 67,557 to 10,500,000. 
We believe those are sufficient to demonstrate the feasibility of our algorithm for distributed computing; for even larger data, 
results from our theoretical analysis (large sample asymptotics) can be used to gain insights.  
\\
\\
Similar as the synthetic data, we conduct experiments under three distributed settings for each UC Irvine dataset.
Table~\ref{tbl:simulationSetting} lists the site specification under such settings. These settings are 
all common in practice. $D_1$ corresponds to situations where different sites have data with disjoint supports,
and $D_2$ the case where data from different sites are mixed in terms of the support of the data, while
$D_3$ is for the case where the data at each site is a random sample from a common data distribution. Each of $D_1$, $D_2$ and $D_3$
in Table~\ref{tbl:simulationSetting} corresponds to two columns with the left indicating the data composition for 
Site 1 and the right for Site 2 (under $D_3$ each of the two distributed sites has a random half of the full data).
The {\it number of data points} at Site 1 and 2, and also {\it the number of representative points} can all be calculated by 
the site specification and the data compression ratio accordingly.
\begin{table}[htb]
\begin{center}
\setlength{\extrarowheight}{2pt}
\begin{tabular}{c|c|c|c|c}
    \hline
\textbf{Data set}     &\textbf{Non-distributed} & $\bm{D_1}$  &$\bm{D_2}$  &$\bm{D_3}$ \\
    \hline
Connect-4   &0.6569 & 0.6576       &0.6569           &0.6569\\
                    & 36       & 17           &15                     &12\\
                    \hline
SkinSeg           &0.9482 &0.9516        &0.9406          &0.9425\\
				&18     &13               &9         &8\\
\hline
USCI           &0.9356 & 0.9382       &0.9404          &0.9396\\
				&215     & 199              &104         &63\\
\hline
Cover type           &0.4984 & 0.4987       &0.4979          &0.4986\\
				&402     & 221              &207         &264\\
\hline
HT Sensor      &0.4960 & 0.5008       &0.4972          &0.4978\\
				&55     & 23              &16         &23\\
\hline
Poker Hand  &0.4977 & 0.4993      &0.4982         &0.5006\\
                     &123      & 67      &65         &62 \\
\hline
Gas Sensor  &0.9865 & 0.9887      &0.9866         &0.9846\\
                     &620      & 287      &258         &254 \\
\hline
HEPMASS  &0.7929 & 0.7949      &0.7927         &0.7902\\
                     &8752      & 3275      &3263         &3321 \\
\hline
\end{tabular}
\end{center}
\caption{\it Clustering accuracy and running time on UC Irvine data under 3 simulation settings, 
$D_1,D_2$ and $D_3$, when K-means clustering is used as the DML. The time 
is in seconds. There are two distributed sites under all of $D_1, D_2$ and $D_3$. } 
\label{tbl:expUCIkMeans}
\end{table}
\\
\\
Table~\ref{tbl:expUCIkMeans} show the clustering accuracy and elapsed time for each UC Irvine dataset under settings 
$D_1, D_2$, and $D_3$, respectively, when K-means clustering is used as DML. The data compression ratios are 200, 
800, 500, 500, 3000, 3000, 16000 and 7000, respectively. Table~\ref{tbl:expUCIrpTrees} is for rpTrees. Due to the random 
nature of rpTrees, it is hard to set the exact data compression ratio. To match the compression ratio for K-means clustering, 
we set the maximum size of leaf nodes to be 200, 800, 500, 500, 3000, 3000, 16000 and 7000, respectively. It can be seen 
that, for all cases, the loss in clustering accuracy due to distributed computing is small. By local computation at individual 
distributed sites in parallel, the overall time required for spectral clustering is significantly reduced. When the computation 
is evenly distributed across different sites, the saving in time is often most substantial; this is typically the case under $D_3$ 
where data are evenly split among two sites. At the similar level of data compression, DML by rpTrees is more efficient in 
computation than that by K-means clustering, but incurs slightly more loss in accuracy. Note that for simplicity, we consider 
two distributed sites for all the UC Irvine data except the HEPMASS data. For the rest of this section, we explore 3 or 4 
distributed sites for the HEPMASS data. 
\begin{table}[htb]
\begin{center}
\setlength{\extrarowheight}{2pt}
\begin{tabular}{c|c|c|c|c}
    \hline
\textbf{Data set}     &\textbf{Non-distributed} & $\bm{D_1}$  &$\bm{D_2}$  &$\bm{D_3}$ \\
    \hline
Connect-4   &0.6577 & 0.6560       &0.6554           &0.6550\\
                    & 7       & 6           &5                     &5\\
                    \hline
SkinSeg           &0.9492 &0.9535        &0.9434          &0.9432\\
				&5     &4               &3         &3\\
\hline
USCI           &0.9394 & 0.9391       &0.9394          &0.9371\\
				&36     & 33              &23         &19\\
\hline
Cover type      &0.4978 & 0.4984       &0.4956          &0.4968\\
				&81     & 50              &52         &53\\
\hline
HT Sensor      &0.4957 & 0.4963       &0.4893          &0.4836\\
				&26     & 10              &7         &7\\
\hline
Poker Hand  &0.4990 & 0.4999      &0.4976         &0.4993\\
                     &47      & 26      &25         &24 \\
\hline
Gas Sensor  &0.9828 & 0.9850      &0.9811         &0.9792\\
                     &305      & 174      &175         &169 \\
\hline
HEPMASS  &0.7906 & 0.7920      &0.7902         &0.7890\\
                     &1568      & 695      &688         &712 \\
\hline
\end{tabular}
\end{center}
\caption{\it Clustering accuracy and running time on UC Irvine data under 3 simulation settings, $D_1,D_2$ and $D_3$, 
when rpTrees is used as the DML. The time is in seconds. There are two distributed sites under all of $D_1, D_2$ 
and $D_3$. } 
\label{tbl:expUCIrpTrees}
\end{table}

\subsubsection{Multiple sites}
To study the impact of the number of distributed sites to our algorithm, we run our algorithm on the HEPMASS data for 3 or 4 
distributed sites. The site configuration is shown in Table~\ref{tbl:settingMultiDist} where, for completeness, we also include 
the case of two sites from the previous section. The results are shown in Table~\ref{tbl:expMultiDist}. We can see that in all 
cases, the clustering accuracy does not or degrade very little when increasing of the number of distributed sites. For both cases, i.e., 
with K-means or rpTrees as the DML, the running time decreases when more sites are involved. However, the decreasing 
trend gradually slows down; such a phenomenon is more pronounced for the case of rpTrees as the local computation takes
much less time than by K-means. This is because, as more distributed sites are 
involved, eventually most of the computation time would be on spectral clustering part. 
\begin{table}[htb]
\begin{center}
\setlength{\extrarowheight}{3pt}
\begin{tabular}{c|c|l}
\hline
\textbf{\# Sites}     & \textbf{Settings}  &\textbf{Site configuration}  \\
    \hline
\multirow{3}{*}{2}   & $D_1$       &Site1: $C_1$, ~Site2: $C_2$            \\\cline{2-3}
      & $D_2$         &Site1: $0.7C_1+0.3C_2$, ~Site2: $0.3C_1+0.7C_2$\\\cline{2-3}
           & $D_3$       &Data randomly partitioned among sites          \\
    \hline
\multirow{4}{*}{3}   & $D_1$       &Site1: $C_1/2$, ~Site2: $C_1/2$, ~Site3: $C_2$            \\\cline{2-3}
      & $D_2$         &Site1: $C_1/2+C_2/4$, ~Site2: $C_1/4+C_2/4$, \\
&&Site3: $C_1/4+C_2/2$	       \\\cline{2-3}
           & $D_3$       &Data randomly partitioned among sites           \\
\hline
\multirow{5}{*}{4}   &$D_1$        &Site1: $C_1/2$, ~Site2: $C_1/2$, ~Site3: $C_2/2$, \\
&&Site4: $C_2/2$             \\\cline{2-3}
  & $D_2$       &Site1: $3/8C_1+C_2/8$, ~Site2: $3/8C1+C_2/8$, \\
&&Site3: $C_1/8+3/8C_2$, ~Site4: $C_1/8+3/8C_2$         \\\cline{2-3}
  &$D_3$		& Data randomly partitioned among sites		 \\
\hline
\end{tabular}
\end{center}
\caption{\it Simulation setting for HEPMASS data with multiple distributed sites. } \label{tbl:settingMultiDist}
\end{table}
\begin{table}[htb]
\begin{center}
\setlength{\extrarowheight}{2pt}
\begin{tabular}{c|c|c|c|c}
    \hline
\textbf{Non-distributed} &\textbf{DML}      & $\bm{D_1}$  &$\bm{D_2}$  &$\bm{D_3}$ \\
    \hline
                     &kmeans$_2$    & 0.7949      &0.7927         &0.7902\\
                     &                         &3275       &3263         & 3321\\\cline{2-5}
   0.7929       &kmeans$_3$    & 0.7937       &0.7913           &0.7904\\
    8752       &                         & 3159           & 2469                    &2106\\\cline{2-5}
                    &kmeans$_4$     &0.7905        &0.7887          &0.7891\\
			   &                          & 1474              &1413         &1465\\
\hline
			   &rptrees$_2$       & 0.7920      &0.7902         &0.7890\\
                    &                          &695       & 688        &712 \\\cline{2-5}
0.7906         &rptrees$_3$       & 0.7895       &0.7900          &0.7866\\
1568	         &                          &   678            & 564        &526\\\cline{2-5}
                   &rptrees$_4$       & 0.7876       &0.7884          &0.7869\\
			  &                           &  387             &396         &401\\
\hline
\end{tabular}
\end{center}
\caption{\it Clustering accuracy and running time (in seconds) on the HEPMASS data under settings, $D_1,D_2$ and $D_3$, 
respectively. The subscripts next to the DML are the number of distributed sites. } 
\label{tbl:expMultiDist}
\end{table}
\section{Conclusion}
\label{section:conclusion}
We have proposed a novel framework that enables spectral clustering for distributed data, with ``minimal" communication 
overhead while a major speedup in computation. 
Our approach is statistically sound in the sense that the achieved accuracy is as good as that when all the data are 
in one place. Our approach achieves computational speedup by deeply compressing the data with DMLs (which also 
sharply reduce the amount of data transmission) and leveraging existing computing resources for local parallel computing 
at individual nodes. The speedup in computation, compared to that in a non-distributed setting, is expected to scale 
linearly (with a potential of even faster when the data is large enough) with the number of distributed nodes when the data 
are evenly distributed across individual sites. Indeed, on all large UC 
Irvine datasets used in our experiments, our approach achieves a speedup of about 2x when there are two distributed 
sites. Two concrete implementations of DMLs are explored, including that by K-means clustering and by rpTrees. Both 
can be computed efficiently, i.e., computation (nearly) {\it linear} in the number of data points. One additional feature of our 
framework is that, as the transmitted data are not in their original form, data privacy may also be preserved. 
\\
\\
Our proposed framework is promising as a general tool for data mining over distributed data. Methods developed under our 
framework will allow practitioners to use potentially much larger data than previously possible, or to attack problems 
previously not feasible, due to the lack of data for various reasons, such as the challenges in big data transmission 
and privacy concerns in data sharing.

\section*{Appendix}
\subsection{The K-means clustering algorithm}
\label{section:appendixKmeans}
Formally, given $n$ data 
points, $K$-means clustering seeks to find a partition of $K$ sets $S_1, S_2, ..., S_K$ such that the 
within-cluster sum of squares, $SS_W$, is minimized
\begin{equation}
\label{equation:kmeansFormu}
\underset{S_1,S_2,...,S_K} {\operatorname{arg\,min}} \sum_{i=1}^{K} \sum_{\mathbf x \in S_i} \left\| \mathbf x - \boldsymbol\mu_i \right\|^2,
\end{equation}
where $\mu_i$ is the centroid of $S_i, i=1,2,...,K$.
\\
\\
Directly solving the problem formulated as \eqref{equation:kmeansFormu} is hard, as it is an integer 
programming problem. Indeed it is a NP-hard problem \cite{ArthurVassilvitskii2006}. 
The K-means clustering algorithm is often referred to a popular implementation sketched as
Algorithm~\ref{alg:kmeans} below. For more details, one can refer to \cite{hartiganWong1979, lloyd1982}.
\begin{algorithm}
\caption{\textbf{$K$-means clustering algorithm}}
\label{alg:kmeans}
\begin{algorithmic}[1]
\STATE Generate an initial set of $K$ centroids $m_1, m_2, ..., m_K$;
\STATE Alternate between the following two steps
\STATE \hspace{\algorithmicindent} Assign each point $x$ to the ``closest" cluster 
	\begin{equation*}     
        \arg \min_{j \in \{1,2,...,K\}} \big \| x-m_j \big \|^2;
        \end{equation*}
\STATE \hspace{\algorithmicindent} Calculate the new cluster centroids
	\begin{equation*}
	m_j^{new} = \frac{1}{\| S_j\|} \sum_{x \in S_j} x, ~~j=1,2,...,K;
        \end{equation*}
\STATE Stop when cluster assignment no longer changes.
\end{algorithmic}
\end{algorithm}
\subsection{The random projection tree algorithm}
\label{section:appendixRPtrees}
in this section, we given an algorithmic description of the generation of rpTrees which
we adopt from \cite{rpForests2018c}.
Let $U$ denote the given data set. Let $t$ denote the rpTree to be built from $U$. Let $\mathcal{W}$ denote the set of 
working nodes. Let $n_T$ denote a predefined constant for the minimal number of data points in a tree node for which 
we will split further. Let $P_{\stackrel{\rightharpoonup}{r}}(x)$ denote the projection coefficient of point $x$ onto line 
$\stackrel{\rightharpoonup}{r}$. Let $\mathcal{N}$ denote the set of neighborhoods, and each element of $\mathcal{N}$ 
is a subset of neighboring points in $U$.
\begin{algorithm}
\caption{\it~~rpTree(U)}
\label{algorithm:neighGen}
\begin{algorithmic}[1]
\STATE Let $U$ be the root node of tree $t$; 
\STATE Initialize the set of working nodes $\mathcal{W} \leftarrow \{U\}$; 
\WHILE {$\mathcal{W}$ is not empty}
	\STATE Randomly pick $W \in \mathcal{W}$ and set $\mathcal{W} \leftarrow \mathcal{W} - \{W\}$; 
	\IF{$|W| < n_T$} 
		\STATE Skip to the next round of the while loop; 
	\ENDIF 
    	\STATE Generate a random direction $\stackrel{\rightharpoonup}{r}$;  
	\STATE Project points in $W$ onto $\stackrel{\rightharpoonup}{r}$, and let $W_{\stackrel{\rightharpoonup}{r}}=\{r \boldsymbol{\cdot} x: x \in W\}$; 
	\STATE Let $a=\min(W_{\stackrel{\rightharpoonup}{r}})$ and $b=\max(W_{\stackrel{\rightharpoonup}{r}})$; 
	\STATE Generate a splitting point $c \sim runif[a,b]$; 
	\STATE Split node $W$ by $W_L=\{x: P_{\stackrel{\rightharpoonup}{r}}(x) < c\}$ and $W_R=\{x: P_{\stackrel{\rightharpoonup}{r}}(x) \geq c\}$; 
	\STATE Update working set by $\mathcal{W} \leftarrow \mathcal{W} \cup \{W_L, W_R\}$; 
\ENDWHILE
\STATE return(t); 
\end{algorithmic}
\end{algorithm}


\end{document}